\documentclass{article}

\PassOptionsToPackage{numbers, compress}{natbib}

\usepackage[preprint]{neurips_2026}

\usepackage[utf8]{inputenc} 
\usepackage[T1]{fontenc}    
\usepackage{hyperref}       
\usepackage{url}            
\usepackage{booktabs}       
\usepackage{amsfonts}       
\usepackage{nicefrac}       
\usepackage{microtype}      
\usepackage{xcolor}         

\title{A Hierarchy of Entropy-Shapley Games for Multivariate Predictive Uncertainty}

\author{%
  Niklas Koenen\\
  Leibniz Institute for Prevention Research and Epidemiology – BIPS\\
  University of Bremen, Germany\\
  \texttt{koenen@leibniz-bips.de}
  \And
  Claudia Battistin\\
  Simula Research Laboratory\\
  Oslo, Norway\\
  \And
  Jeriek Van den Abeele\\
  Telenor Research \& Innovation\\
  Fornebu, Norway\\
  \And
  Martin Jullum\\
  Norwegian Computing Center\\
  Oslo, Norway\\
}

\usepackage{bm}
\usepackage{tikz}
\usepackage{amsmath, amsthm, amssymb}
\usepackage{wrapfig}
\usepackage{lipsum}
\usepackage{caption}
\usepackage{subcaption}

\usetikzlibrary{arrows.meta, decorations.pathreplacing}

\definecolor{cL1}{HTML}{2E86C1}
\definecolor{cL2}{HTML}{E67E22}
\definecolor{cL3}{HTML}{E74C3C}
\definecolor{cGray}{HTML}{CCCCCC}

\newtheorem{theorem}{Theorem}

\newtheorem{proposition}[theorem]{Proposition}
\newtheorem{corollary}[theorem]{Corollary}
\theoremstyle{definition}

\renewcommand{\mid}{\,|\,}

\begin{document}

\maketitle

\begin{abstract}
Modern probabilistic machine learning models increasingly produce multivariate outputs with complex dependence structure, from multi-step time-series forecasts to sample path predictions. Understanding which input features drive the predictive uncertainty is important for risk-aware decisions, model diagnostics, and deciding whether the uncertainty should be mitigated or hedged against. This attribution problem requires a choice of how dependencies between output components are treated. Existing approaches reduce the output to a scalar through aggregation or projection before attribution, thereby obscuring whether features affect marginal uncertainty, dependence structure, or both, while component-wise analyses can miss dependence effects entirely. We close this gap by introducing a hierarchy of three entropy-based Shapley games that make this output-side choice explicit for any ordered multivariate outcome, ranging from per-component marginal entropy to fully joint entropy. The hierarchy isolates a cross-component attribution term that captures how each feature shifts the dependence between output components, a quantity invisible to component-wise methods. We establish a chain-rule decomposition of the joint attribution and characterize the cross-component term through conditional total correlation, providing both closed-form and sample-based estimators. Finally, we demonstrate how the framework captures differences in learned joint structure across probabilistic models from distributional regression to a zero-shot time series foundation model.
\end{abstract}

\section{Introduction}\label{sec:intro}

Over the last decade, machine learning has shifted focus from predicting low-dimensional targets, e.g., a few class labels or scalar responses, to producing or even generating multivariate outputs with rich dependence structure. For example, time-series models generate multi-step forecasts~\cite{salinas2020deepar, Zhou2021}, modern language models output text as token sequences~\cite{brown2020gpt3, touvron2023llama}, and prediction models in autonomous systems generate paths over subsequent time steps~\cite{salzmann2020trajectron, varadarajan2022multipath}. Such targets $\bm{Y} = (Y_1, \ldots, Y_T)$ are typically \emph{multivariate} with a natural ordering, where the components exhibit complex dependencies. In a day-ahead weather forecast, for example, the temperature at noon naturally depends on the preceding morning temperature, and capturing this dependence is part of the modelling task.

The shift toward more nuanced outcomes is reflected in a wave of probabilistic model classes that no longer produce multiple point predictions alone, but a \emph{joint predictive distribution} or \emph{samples} over all output components in $\bm{Y}$, as multiple realizations are typically plausible and should be viewed in conjunction (e.g., different weather scenarios). These models range from distributional regression~\cite{duan2020ngboost, rasp2018neural} and probabilistic forecasting models~\cite{salinas2020deepar} to foundation models~\cite{ansari2024chronos, woo2024moirai, das2024timesfm}, which even provide multi-step zero-shot forecasts. In such settings, the realizations' variability and predictive uncertainty are typically the quantity of interest, since the joint dependence structure of the output makes point predictions an insufficient basis for downstream decisions that must account for the full range of plausible outcomes. Hence, this uncertainty is structured, combining per-component variability with the dependence across components.

As such models are increasingly deployed in safety-critical or sensitive applications, understanding the origin of predictive uncertainty becomes as important as quantifying it. For instance, if a multi-step forecast exhibits high uncertainty, a planner needs to know whether this stems from an inherently volatile input feature (which must be hedged against) or a localized lack of historical data, which might be mitigated.
For multivariate outputs $\bm{Y}$, the attribution of uncertainty inherits this structure: features may drive predictive uncertainty in specific output components or in the dependence between them. In the weather example, this is the difference between asking which inputs drive uncertainty about the temperature at noon, and which drive the joint behavior of the noon and 1 p.m. predictions. 
Thus, explaining predictive uncertainty is not merely a diagnostic exercise: it helps determine whether uncertainty should be mitigated, monitored, or incorporated into downstream decisions.

Shapley-based feature attribution~\cite{lundberg2017, sundararajan2020many, chen2023algorithms} provides a game-theoretical framework which can be used to answer such questions by assigning feature contributions with respect to a suitable value function. Natural choices for uncertainty include predictive variance~\cite{gajewski2025varshap, Chiaburu2025} and entropy~\cite{watson2023explaining}. While these are well-defined for scalar outputs, their extension to multivariate outputs is not uniquely defined, as it requires specifying how uncertainty is aggregated or conditioned across components. 
Even for non-uncertainty attribution, existing methods typically either operate component-wise on each $Y_t$~\cite{bento2021timeshap, FrancoDeLaPea2026, nayebi2023windowshap, Nguyen2025}, ignoring dependencies between output components, or aggregate the multivariate output before attribution~\cite{zhang2024shaptime, jethani2023dontbefooled}, thereby obscuring where in the output the effect arises. 
This limitation is not merely a modeling choice, but a mathematical necessity. Recent work has proven that if a vector-valued attribution method preserves the classical Shapley axioms, it must evaluate each output component independently~\cite{biccari2026fair}.
Concretely, features that affect only the dependence structure of $\bm{Y}$ would receive an attribution of zero (see Sec.~\ref{sec:experiments-poc} for an example), exposing a structural blind spot in component-wise attribution. This leaves open how to define Shapley-based uncertainty attributions for ordered multivariate outputs in a way that separates marginal uncertainty from dependence among output components.

To close this gap, we make the choice of the value function on the output side explicit and introduce a \emph{hierarchy of entropy-based Shapley games} that differ in how they treat dependencies across output components. The hierarchy separates different structural aspects of predictive uncertainty, ranging from marginal variability to fully joint dependence. While instantiated in this paper for multi-step forecasting, the framework applies to any ordered multivariate output.

\paragraph{Contributions.}
\textbf{(1)} 
To the best of our knowledge, we introduce the first framework for multivariate predictive uncertainty attribution, proposing a hierarchy of entropy-based value functions that consistently extends scalar uncertainty attribution to multivariate outputs by making the output-side marginalization choice explicit. The hierarchy resolves a structural blind spot of standard component-wise attribution, which assigns an attribution of exactly zero to features that drive only cross-component dependence. \textbf{(2)} We establish two propositions linking the levels via a chain-rule decomposition (Prop.~\ref{prop:chain-rule-linkage}) and a total-correlation-based characterization of cross-component effects (Prop.~\ref{prop:cross-component}). \textbf{(3)} We provide closed-form expressions for the multivariate Gaussian case and discuss model-agnostic sample-based estimators for the general case, both validated against analytical baselines. \textbf{(4)} We demonstrate the framework on a range of probabilistic models, from NGBoost~\citep{duan2020ngboost} and DeepAR~\citep{salinas2020deepar} to the zero-shot foundation model Chronos-T5~\cite{ansari2024chronos}.

\section{Related Work}\label{sec:related_work}

Feature attribution is a central tool in explainable AI, with well-established post-hoc methods for assigning \emph{local}~\cite{lundberg2017, Ribeiro2016, Sundararajan2017} or \emph{global}~\cite{covert2020sage, Fisher2019} contributions to features. Various explanation targets have been considered, including the model prediction, its risk, and sensitivity~\cite{fumagalli2025unifying}. In this work, we focus on Shapley-based attributions of predictive uncertainty, in particular, the entropy. We organize the related work along two axes: \emph{what} is being attributed and \emph{how} the output structure is treated.

\textbf{Feature Attribution for Predictive Uncertainty.}
While some works consider feature-based explanations of predictive uncertainty via global risk-based approaches, e.g., conformal predictions~\cite{Mehdiyev2025} or entropy-based PFI/PDP variants~\cite{Wood2024}, we focus on Shapley-based attributions. Within the Shapley framework, the choice of the value function determines which aspect of the predictive uncertainty is attributed: residuals, predictive likelihood, and cross-entropy for risk-based approaches~\cite{Lundberg2020fromlocal, chen2018lshapley}, prediction-interval widths in conformal settings~\cite{idrissi2025unveil}, local prediction variance under input perturbations~\cite{gajewski2025varshap, Chiaburu2025}, or global variance contributions via the connection between Sobol' indices and Shapley values~\cite{Owen2014}. Closest to our setting are information-theoretic approaches, which directly leverage the domain of quantifying information and uncertainty. SHAP-KL~\cite{jethani2023dontbefooled} attributes the predictive distribution shift using the KL divergence, while InfoSHAP~\cite{watson2023explaining} uses the conditional entropy $H(Y \mid \bm{X})$ as a value function. Common to all of these approaches, however, is that uncertainty is reduced to a single scalar quantity. As a consequence, they cannot resolve how a feature contributes to different structural components of the predictive distribution, a limitation that becomes particularly relevant when the prediction itself is multivariate.

\textbf{Feature Attribution for Multivariate Outputs.}
Existing attribution approaches for multivariate outputs $\bm{Y} = (Y_1, \dots, Y_T)$ either reduce $\bm{Y}$ to a scalar before attribution by aggregating over components~\cite{zhang2024shaptime, jethani2023dontbefooled}, or projecting onto a single one~\cite{Krzyziski2023, langbein2025gradientbased}. A parallel literature on variance-based sensitivity analysis develops Sobol' indices for vector-valued outputs~\cite{Gamboa2014, Heredia2022}, but provides only global rather than local attributions. Approaches that operate on the predictive distribution face a similar restriction: \citet{tonekaboni2020} attribute the distribution shift via a KL-based decomposition, and ConformaSegment~\cite{conformasegment2025} attributes the shift in conformal interval bounds, both on scalar or per-step outputs. In this field, most approaches address sequential \emph{inputs} for scalar predictions rather than dependencies within the output and are orthogonal to our work~\cite{bento2021timeshap, nayebi2023windowshap, Nguyen2025, FrancoDeLaPea2026}. Beyond the time series domain, two recent works address the multivariate-output question more directly. Shapley Chains~\cite{Ayad2022} extend Shapley values to multi-output classification via chains, and \citet{biccari2026fair} prove a rigidity result stating that any attribution rule satisfying the classical Shapley axioms on vector-valued games must decompose component-wise across outputs. Neither addresses predictive \emph{uncertainty} nor formalizes local feature contributions to the dependency structure between output components.

To the best of our knowledge, no existing approach attributes predictive uncertainty for multivariate dependent outputs in a way that resolves both marginal feature contributions and contributions to the dependence structure. Our hierarchy of entropy-based Shapley games closes this gap.

\section{Background}\label{sec:background}

Throughout this paper, lowercase letters denote scalars (e.g., $x \in \mathbb{R}$), boldface lowercase denotes vectors (e.g., $\bm{x} \in \mathbb{R}^p$), uppercase denotes a random variable (e.g., $X \sim \mathcal{N}(0,1)$), and boldface uppercase denotes random vectors or matrices (e.g., $\bm{X} \sim \mathcal{N}(\bm{\mu}, \bm{\Sigma})$). Calligraphic letters denote a space or the support of the corresponding random variable (e.g., $X$ takes values in $\mathcal{X}$). 
Assuming all relevant probability densities exist, we use the generic notation $p(\cdot)$ throughout, relying on the arguments to specify the respective random variables (e.g., $p(\bm{y} \mid \bm{x})$).
Additionally, we write $[p] := \{1, \ldots, p\}$ for the set of feature indices, $\bar{S} := [p] \setminus S$ for the complement of $S \subseteq [p]$, $2^{[p]}$ for its power set, and $\bm{X}_S := (X_i)_{i \in S}$ for the components of $\bm{X}$ indexed by $S$, and $\bm{X}_{<t} := (X_1, \dots, X_{t-1})$ for the history of output components preceding index $t$.

\subsection{Information Theory}\label{sec:background-entropy}

To attribute predictive uncertainty to individual features, we require a quantity that captures the full distributional uncertainty of a (possibly multivariate) model prediction. Information theory provides a clear answer: the \textit{entropy}. The (differential) entropy $H(X)$ of a random variable $X$ on $\mathcal{X}$ with density $p$ measures the uncertainty over its realizations and is defined as
\begin{align}
    H(X) = \mathbb{E}_{X}\!\bigl[-\log p(X)\bigr] 
    = -\int_{\mathcal{X}} p(x)\,\log p(x)\;dx.
    \label{eq:entropy}
\end{align}
%
Low entropy indicates a predictable outcome and, conversely, high entropy indicates high uncertainty. For example, a fair coin toss with $\mathbb{P}[\text{Heads}] = \tfrac{1}{2}$ has entropy $H = -2 \cdot \tfrac{1}{2}\log\tfrac{1}{2} = \log 2 \approx 0.693$ nats, reflecting high uncertainty, while a heavily biased coin with $\mathbb{P}[\text{Heads}] = 0.99$ has near-zero entropy, since the outcome is almost predetermined.

Closely related is the conditional entropy. Given a specific observation $c ~\sim C$, the \textit{local conditional entropy} $H(X \mid C = c)$ measures the residual uncertainty in $X$ for that particular realization. Averaging over all possible values of $C$ yields the \textit{(global) conditional entropy} $H(X \mid C) = \mathbb{E}_{\tilde{C}}[H(X \mid C = \tilde{C})]$, which captures the expected residual uncertainty in $X$ after knowing $C$. Returning to the coin toss example, suppose a friend $C$ selects the fair or the biased coin with equal probability, and only he knows which coin is tossed. If he reveals that he used the fair coin, then we are more uncertain with $H(X \mid C = \text{fair}) = \log 2 \approx 0.693$ nats, while $H(X \mid C = \text{biased}) \approx 0.06$ in the biased case. The global conditional entropy $H(X \mid C) \approx 0.375$ averages over these two values, reflecting how much overall uncertainty about the outcome remains once we know which coin is tossed. The reduction in uncertainty about $X$ from observing $C$ is also known as the \textit{mutual information} $I(X; C) := H(X) - H(X \mid C)$. It quantifies the dependence of both variables and satisfies $I(X; C) \geq 0$ with equality iff $X$ and $C$ are independent.

These definitions extend naturally to random vectors: for $\bm{X}$ on $\mathcal{X} \subseteq \mathbb{R}^p$, the \textit{joint entropy} $H(\bm{X})$, the conditional entropy $H(\bm{X} \mid  \bm{C})$ and the mutual information $I(\bm{X}; \bm{C}) = H(\bm{X}) - H(\bm{X} \mid \bm{C})$ are defined analogously with the integral taken over $\mathcal{X}$ using the joint or conditional distribution, respectively. 
The dependence among the components of a random vector $\bm{X}$ is measured by the \textit{total correlation} $\operatorname{TC}(\bm{X}) := \sum_{i=1}^{p} H(X_i) - H(\bm{X})$. Like mutual information, $\operatorname{TC}(\bm{X}) \geq 0$ with equality iff the components of $\bm{X}$ are mutually independent, and it reduces to $I(X_1; X_2)$ for $p = 2$. A key identity linking joint and conditional entropies is the \textit{chain rule of entropy},
\begin{align}
    H(\bm{X}) = \sum_{i=1}^{p} H(X_i \mid X_1, \ldots, X_{i-1}) = \sum_{i=1}^{p} H(X_i \mid \bm{X}_{<i}),
    \label{eq:chain-rule}
\end{align}
which decomposes joint uncertainty into sequential increments. Each summand captures the residual uncertainty in $X_i$ given all preceding components. The identity extends to both local and global conditional entropies, e.g., $H(\bm{X} \mid \bm{C}) = \sum_{i=1}^{p} H(X_i \mid X_1, \ldots, X_{i-1}, \bm{C})$. For more details on information-theoretic background, we refer to App.~\ref{app:information-theory}.

\subsection{Shapley Values}\label{sec:background-shapley}

Shapley values \citep{Shapley1953} provide a fair attribution of a total \textit{"payoff"} among players in a cooperative game. For feature attribution in machine learning \citep{lundberg2017}, the players are typically the input features $X_1, \ldots, X_p$ and a value function $v \colon 2^{[p]} \times \mathcal{X} \to \mathbb{R}$ assigns each \textit{coalition} $S \subseteq [p]$ the corresponding real-valued payoff for a specific instance $\bm{x} \in \mathcal{X}$\footnote{We only consider local attributions in this paper. However, there also exist global Shapley attributions \citep{covert2020sage, Owen2014}.}. The Shapley value of feature $j$ is its average marginal contribution across all possible coalitions:
\begin{align}
    \phi_v(j, \bm{x}) = \sum_{S \subseteq [p] \setminus \{j\}} \frac{|S|!\,(p - |S| - 1)!}{p!} \Bigl[v(S \cup \{j\}, \bm{x}) - v(S, \bm{x})\Bigr].
    \label{eq:shapley}
\end{align}
For any value function $v$, this is the unique attribution satisfying efficiency, symmetry, linearity, and the null player axioms~\citep{lundberg2017}.

A key feature of the Shapley framework is that different value functions $v$ yield different attributions, each answering another question. For instance, standard SHAP \citep{lundberg2017} defines the value function for a predictive model $f: \mathcal{X} \to \mathcal{Y}$ as $v_0(S, \bm{x}) = \mathbb{E}[f(\bm{X}) \mid \bm{X} = (\bm{x}_S, \bm{X}_{\bar{S}})]$. The resulting Shapley values decompose the prediction into feature-wise contributions relative to the average prediction by the efficiency property, i.e., $f(\bm{x}) - \mathbb{E}[f(\bm{X})] = \sum_{j=1}^p \phi_{v_0}(j, \bm{x})$. Other choices include (local and global) risk-based value functions defined via expected loss \citep{Lundberg2020fromlocal,covert2020sage}, sensitivity-based approaches~\cite{Owen2014}, and information-theoretic formulations based on divergences such as the KL divergence \citep{jethani2023dontbefooled}. However, closely related to our approach is InfoSHAP \citep{watson2023explaining}, which moves from explaining predictions to explaining \emph{predictive uncertainty} by using the conditional entropy as the value function,
\begin{align}
    v_{H}(S, \bm{x}) = \mathbb{E}_{\bm{X}_{\bar{S}}} \bigl[H\bigl(Y \bigm|  \bm{X} = (\bm{x}_S, \bm{X}_{\bar{S}})\bigr)\bigr],
    \label{eq:infoshap}
\end{align}
so that the resulting Shapley values decompose how much each feature contributes to the predictive uncertainty about the outcome $Y$. InfoSHAP is formulated and evaluated for scalar outcomes $Y \in \mathbb{R}$ and has not been extended to multivariate targets $\bm{Y}$ including their dependence structure.

\section{Multivariate Entropy-Shapley Attribution}\label{sec:method}

When explaining the predictive uncertainty of a model with multivariate output $\bm{Y} = (Y_1, \dots, Y_T)$, a fundamental choice arises: should we explain the uncertainty of each component separately, or of the full vector jointly? The first ignores dependencies between components, whereas the second captures these dependencies but loses information about where in the output the uncertainty is located. Neither alone tells the full story, which motivates our \emph{hierarchy of three entropy-based Shapley games} that we formalize below, each defining a value function based on the conditional entropy but evaluated at a different resolution of the output vector. While the framework applies to any multivariate output with a natural order of components, we instantiate it on multi-step time series forecasting in Sec.~\ref{sec:experiments}. 

\subsection{Hierarchy of Entropy Games}\label{sec:method-hierarchy}

All three game levels share the same player set $[p]$ and extend the InfoSHAP value function Eq.~\eqref{eq:infoshap} to multivariate outputs, but differ in which aspect of the joint distribution $p(\bm{y} \mid \bm{x})$ they target.

\textbf{Level 1 (Marginal).} The first game applies InfoSHAP \citep{watson2023explaining} to each $Y_t$ separately to explain the marginal uncertainty at each output component $t$, while ignoring the dependencies within the outcome:
\begin{align}
    v_{H}^{(t)}(S, \bm{x}) := \mathbb{E}_{\bm{X}_{\bar{S}}} \bigl[H\bigl(Y_t \bigm|  \bm{X} = (\bm{x}_S, \bm{X}_{\bar{S}})\bigl)\bigr], \qquad t = 1, \ldots, T.
    \label{eq:level1}
\end{align}%
By decomposing the difference between the local and global conditional entropies $H(Y_t \mid \bm{x}) - H(Y_t \mid \bm{X})$, these Shapley values informally quantify \textit{how much each feature contributes to the predictive uncertainty of component $t$, compared to the population average.}

\textbf{Level 2 (Sequential).} The second game explains the incremental uncertainty at component $t$, given all preceding components:
\begin{align}
    v_{H}^{(t \mid <t)}(S, \bm{x}) :=  \mathbb{E}_{\bm{X}_{\bar{S}}} \bigl[H\bigl(Y_t \bigm| \bm{Y}_{<t}, \bm{X} = (\bm{x}_S, \bm{X}_{\bar{S}})\bigl)\bigr], \qquad t = 1, \ldots, T.
    \label{eq:level2}
\end{align}
Each game corresponds to one term in the chain rule decomposition in Eq.~\eqref{eq:chain-rule}, measuring the residual uncertainty in $Y_t$ that remains after conditioning on earlier outcome components and the available features. 
By decomposing the difference between the sequential local and global conditional entropies, $H(Y_t \mid \bm{Y}_{<t}, \bm{x}) - H(Y_t \mid \bm{Y}_{<t}, \bm{X})$, the associated Shapley values informally quantify \textit{how much each feature contributes to the additional predictive uncertainty at $t$, beyond what the preceding components already reveal, compared to the population average.}

\textbf{Level 3 (Joint).} The third game explains the joint uncertainty of the full output vector:
\begin{align}
    v_{H}^\text{joint}(S, \bm{x}) :=  \mathbb{E}_{\bm{X}_{\bar{S}}} \bigl[H\bigl(\bm{Y} \bigm|  \bm{X} = (\bm{x}_S, \bm{X}_{\bar{S}})\bigl)\bigr].
    \label{eq:level3}
\end{align}
This captures all marginal and cross-component effects simultaneously, but no longer reveals where in the output the uncertainty is located. 
By decomposing the difference between the full joint local and global conditional entropies, $H(\bm{Y} \mid \bm{x}) - H(\bm{Y} \mid \bm{X})$, the associated Shapley values informally quantify \textit{how much each feature contributes to the predictive uncertainty of the entire forecast, compared to the population average.}

Substituting the value functions into Eq.~\eqref{eq:shapley} yields the Level 1--3 Shapley values: $\phi^{(t)}(j, \bm{x})$, $\phi^{(t \mid <t)}(j, \bm{x})$, and $\phi^{\text{joint}}(j, \bm{x})$. 
As noted above, these decompose the respective local--global conditional entropy gaps (see also Corollary~\ref{cor:local-global-entropy-decomposition}). 
The specific global conditional entropy contrasts arise because the value functions in Eqs.~\eqref{eq:level1}--\eqref{eq:level3} rely on \textit{marginalizations} of the global conditional entropy, rather than a subset-conditional approach (which for Level~3 would correspond to $H(\bm{Y} \mid \bm{X}_S = \bm{x}_S)$ as value function, and decomposing $H(\bm{Y} \mid \bm{x}) - H(\bm{Y})$). 
As \citet{watson2023explaining} note, however, this demands intractable coalition densities $p(\bm{y} \mid \bm{x}_S)$, while our formulation only requires the full conditional $p(\bm{y} \mid \bm{x})$. 
Consequently, in our games, $\phi(j, \bm{x}) > 0$ indicates that feature $j$ drives higher predictive uncertainty than the \textit{population average}, while $\phi(j, \bm{x}) < 0$ implies a reduction.

The resulting Shapley values for the three levels are not independent constructions, as they are linked through the chain rule of entropy Eq.~\eqref{eq:chain-rule}. The proof can be found in App.~\ref{app:proofs}.
\begin{proposition}[Chain-rule linkage]\label{prop:chain-rule-linkage}
The Level~3 Shapley values decompose additively into the Level~2 Shapley values across components, i.e., for all $j \in [p]$ it holds that
\begin{align}
    \phi^{\text{joint}}(j, \bm{x}) = \sum_{t=1}^{T} \phi^{(t \mid <t)}(j, \bm{x}).
    \label{eq:chain-rule-linkage}
\end{align}
Moreover, if the output components $Y_1, \ldots, Y_T$ are conditionally independent given $\bm{X}$, then $\phi^{(t \mid <t)}(j, \bm{x}) = \phi^{(t)}(j, \bm{x})$ for all $t$, and Level~3 reduces to the sum of Level~1 values.
\end{proposition}
This implies that Level~2 Shapley values provide a component-wise decomposition of the joint attribution. If a feature drives overall forecast uncertainty (Level~3), Level~2 reveals at which components this effect materializes and how strongly.


While Prop.~\ref{prop:chain-rule-linkage} connects Level~2 and Level~3, the relationship between Level~1 and Level~3 is controlled by the dependency structure of the outcomes. To formalize this, we introduce the total correlation game, defined by the following value function:
\begin{align}
   v_{\operatorname{TC}}(S, \bm{x}) 
:= \mathbb{E}_{\bm{X}_{\bar{S}}}[\operatorname{TC}(\bm{Y} \mid \bm{X}) \mid \bm{X} = (\bm{x}_S, \bm{X}_{\bar{S}})].
\end{align}
The associated Shapley values $\phi^{\operatorname{TC}}(j, \bm{x})$ constitute the cross-component attributions by decomposing the deviation of the local total correlation from its population average $\operatorname{TC}(\bm{Y} \mid \bm{x}) - \operatorname{TC}(\bm{Y} \mid \bm{X})$ (see App.~\ref{app:information-theory} for further details on $\operatorname{TC}$). 
The below Proposition highlights how the TC game connects the Level~1 and Level~3 games:
\begin{proposition}[Cross-component decomposition]\label{prop:cross-component}
The Shapley values of the TC game isolate the contribution of feature $j$ from the dependency structure between output components, i.e., for all $j \in [p]$ it holds that $\phi^{\operatorname{TC}}(j, \bm{x}) = \sum_{t = 1}^T \phi^{(t)}(j, \bm{x}) - \phi^{\text{joint}}(j, \bm{x})$.
%
\end{proposition}

The total correlation measures the overall statistical redundancy between output components. Prop.~\ref{prop:cross-component} thus gives the cross-component attributions a precise interpretation: they decompose how much each feature shifts the inter-component dependence relative to a typical instance. Crucially, features that only affect the dependence structure without changing any marginal distribution receive $\phi^{(t)}(j, \bm{x}) = 0$ for all $t$, making them entirely invisible to Level~1 (see Sec.~\ref{sec:experiments-poc} for an example).

\subsection{Estimating the Value Functions}\label{sec:method-estimation}


Evaluating the value functions (Eqs.~\ref{eq:level1}--\ref{eq:level3}) in practice requires decomposing them into two operations: an outer expectation over the out-of-coalition features $\bm{X}_{\bar{S}}$, and an inner conditional entropy evaluated at a specific \textit{completed input} $\tilde{\bm{x}}$. The outer expectation represents the standard Shapley imputation step, typically approximated via Monte Carlo integration by drawing $K$ background samples $\bm{x}_{\bar{S}}^{(k)}$ to form $K$ completed inputs $\tilde{\bm{x}}^{(k)} = (\bm{x}_S, \bm{x}_{\bar{S}}^{(k)}), k = 1,\ldots, K$. Our hierarchy is agnostic to how these are drawn; any standard imputation method (e.g., baseline, marginal, or conditional \citep{aas2021explaining,frye2020shapley,chen2023algorithms}) can be used. Consequently, the core technical challenge specific to our method reduces to evaluating the inner entropies (e.g., $H(\bm{Y} \mid \bm{X} = \tilde{\bm{x}}^{(k)})$) for a fixed completed input $\tilde{\bm{x}}^{(k)}$. How we estimate these terms depends on how the model provides access to $p(\bm{y} \mid \tilde{\bm{x}})$. We consider three cases: closed-form joint densities, illustrated by multivariate Gaussian outputs (Sec.~\ref{sec:method-estimation-gaussian}); factorized one-step conditionals, as in autoregressive models (Sec.~\ref{sec:method-estimation-conditional}); and purely sample-based outputs (Sec.~\ref{sec:method-estimation-samples}).

\subsubsection{Multivariate Gaussian Outputs}\label{sec:method-estimation-gaussian}

A common class of probabilistic models directly learns distributional parameters, e.g., $\mu(\bm{x})$ and $\bm{\Sigma}(\bm{\tilde{x}})$ for a multivariate Gaussian outcome $\bm{Y} \mid \bm{\tilde{x}} \sim \mathcal{N}(\bm{\mu}(\bm{\tilde{x}}), \bm{\Sigma}(\bm{\tilde{x}}))$. In this case, the value function at each level of the hierarchy (Sec.~\ref{sec:method-hierarchy}) admits a closed-form expression that depends solely on the covariance matrix $\bm{\Sigma}(\bm{\tilde{x}})$. For more details on the entropy of multivariate Gaussians, we refer to \cite{Cover2005}.

For Level~1, the entropy of each component $Y_t$ depends only on the corresponding diagonal entry:
\begin{align}
    H(Y_t \mid \bm{\tilde{x}}) = \tfrac{1}{2} \log \left(2\pi e\, \Sigma_{tt}(\bm{\tilde{x}}) \right), \qquad t = 1, \ldots, T.\label{eq:gaussian-l1}
\end{align}
For Level~2, the conditional distribution $Y_t \mid \bm{Y}_{<t}, \bm{\tilde{x}}$ is again Gaussian, with variance given by the Schur complement of the leading $(t{-}1) \times (t{-}1)$ submatrix in the leading $t \times t$ submatrix of $\bm{\Sigma}(\bm{\tilde{x}})$:
\begin{align}
    H(Y_t \mid \bm{Y}_{<t}, \bm{\tilde{x}}) = \tfrac{1}{2}\log\bigl(2\pi e\, \sigma^2_{t \mid <t}(\bm{\tilde{x}})\bigr), \quad \text{where} \quad \sigma^2_{t \mid <t} = \Sigma_{tt} - \bm{\Sigma}_{t,<t}\, \bm{\Sigma}_{<t,<t}^{-1}\, \bm{\Sigma}_{<t,t}.
    \label{eq:gaussian-l2}
\end{align}
For Level~3, the entropy follows from the determinant of the full covariance matrix $\bm{\Sigma}(\bm{\tilde{x}})$:
\begin{align}
    H(\bm{Y} \mid \bm{\tilde{x}}) = \tfrac{1}{2}\log \left((2\pi e)^T\det\bm{\Sigma}(\bm{\tilde{x}}) \right).
    \label{eq:gaussian-l3}
\end{align}
These three formulas show that the value function is entirely independent of the mean prediction $\bm{\mu}(\bm{x})$, i.e., features that only affect the mean receive zero entropy attribution at every level. Each level operates on a different part of the covariance matrix (see Fig.~\ref{fig:cov-decomposition}) and thus reflects how each level treats the output dependencies (Sec.~\ref{sec:method-hierarchy}). The same closed-form computations also extend to other distributional families with tractable marginal, conditional, and joint entropies, e.g., to multivariate Student-$t$ or, more generally, elliptical distributions~\cite{arellano2013, Fang2018}.

\subsubsection{Factorized Conditional Densities}\label{sec:method-estimation-conditional}

A second class of probabilistic models, including autoregressive forecasting models such as DeepAR~\cite{salinas2020deepar}, factorizes the predictive distribution along the natural ordering of $\bm{Y} = (Y_1, \ldots, Y_T)$ as $p(\bm{y} \mid \bm{\tilde{x}}) = \prod_{t=1}^{T} p(y_t \mid \bm{y}_{<t}, \bm{\tilde{x}})$ by learning the one-step conditionals $p(y_t \mid \bm{y}_{<t}, \bm{\tilde{x}})$, typically as a parametric distribution (e.g., Gaussian, Student-$t$ or log-normal). The factorization yields Levels~2 and~3 directly, while Level~1 requires additional estimation. The Level~2 entropies $H(Y_t \mid \bm{Y}_{<t}, \bm{\tilde{x}})$ are obtained by averaging the closed-form entropy of the one-step conditional $p(y_t \mid \bm{y}_{<t}, \bm{\tilde{x}})$ over samples of $\bm{Y}_{<t}$ drawn along the factorization, and the Level~3 entropy follows by the chain rule (Prop.~\ref{prop:chain-rule-linkage}). The Level~1 marginal entropies $H(Y_t \mid \bm{\tilde{x}})$, by contrast, have no closed form, since the marginal $p(y_t \mid \bm{\tilde{x}})$ is a mixture of one-step conditionals. They can be estimated from samples through marginalization, e.g., by Monte-Carlo averaging the conditional density $p(y_t \mid \bm{y}_{<t}, \bm{\tilde{x}})$ over samples of $\bm{Y}_{<t}$, or by applying a one-dimensional nonparametric entropy estimator to the marginal samples of $Y_t$, which avoids high-dimensional density estimation.

\subsubsection{Sample-based Entropy Estimation}\label{sec:method-estimation-samples}


When $p(\bm{y}\mid \bm{x})$ is unavailable, we estimate the local entropies from simulated joint trajectories. For each completed input $\tilde{\bm{x}}$ arising in outer expectation step (Sec.~\ref{sec:method-estimation}), we draw trajectories $\bm{y}^{(1)},\ldots,\bm{y}^{(N)}\sim p(\bm{Y}\mid \tilde{\bm{x}})$ and evaluate the required entropy terms using one of several approaches.

\textbf{Nonparametric kNN estimation.}
A fully nonparametric alternative uses nearest-neighbour distances \citep{kozachenko1987sample}. While this efficiently yields Level~1 marginals, Level~2 sequential terms require computing differences 
$H(\bm{Y}_{\leq t}\mid \bm{X}=\tilde{\bm{x}})-H(\bm{Y}_{<t}\mid \bm{X}=\tilde{\bm{x}})$,
from which Level~3 follows. Although assumption-light, computing potentially high-dimensional joint entropies for larger horizons at Level~2 and 3  suffers from the curse of dimensionality, demanding a large sample size $N$ to maintain accuracy.

\textbf{Parametric Gaussian estimation.} 
Conversely, assuming a multivariate Gaussian predictive distribution allows efficient entropy evaluation by applying the closed-form expressions (Sec.~\ref{sec:method-estimation-gaussian}) directly to the sample covariance matrix $\hat{\bm{\Sigma}}(\tilde{\bm{x}})$ estimated from the trajectories. However, while computationally cheap, this estimator remains inherently biased for predictive distributions exhibiting non-Gaussian characteristics such as skewness or heavy tails.

\textbf{Semiparametric Gaussian copula.}
To handle non-Gaussian marginals while retaining computational efficiency, we employ a Gaussian copula. For each completed input $\tilde{\bm{x}}$, we estimate one-dimensional marginal densities $\hat{f}_t$ and CDFs $\hat{F}_t$ by KDE, and use the fitted marginals to compute the Level~1 entropies $\widehat H(Y_t\mid \tilde{\bm{x}})$. The samples are then transformed to latent Gaussian variables via $z_t^{(i)}=\Phi^{-1}(\hat F_t(y_t^{(i)}\mid\tilde{\bm{x}}))$, from which we estimate the empirical correlation matrix $\hat{\bm{R}}(\tilde{\bm{x}})$. Level~2 sequential entropies are computed as
$\widehat H(Y_t\mid \bm{Y}_{<t},\tilde{\bm{x}}) = \widehat H(Y_t\mid \tilde{\bm{x}}) + \frac{1}{2}\log(\hat{\sigma}^2_{t\mid <t,z})$,
where $\hat{\sigma}^2_{t\mid <t,z}$ is the Schur complement of $\hat{\bm{R}}$, and Level~3 follows by Prop.~\ref{prop:chain-rule-linkage}. The estimator thus allows flexible marginals while retaining stable and tractable Gaussian-copula dependence; full formulae are given in App.~\ref{app:copula}.

\section{Experiments}\label{sec:experiments}

We evaluate the proposed hierarchy in three steps, moving from controlled closed-form settings to sample-based estimation and model comparison. We first use a synthetic Gaussian DGP (Sec.~\ref{sec:experiments-poc}) and a real distributional-regression model (Sec.~\ref{sec:experiments-bikesharing}), where the entropy terms are available analytically.  Finally, we apply our hierarchy to the UCI Electricity dataset \cite{trindade2015electricity}, comparing sample-based cross-component share on DeepAR \cite{salinas2020deepar} and Chronos \cite{ansari2024chronos}. The code for reproduction is available on \href{https://github.com/bips-hb/paper-shapley-hierarchy}{GitHub}\footnote{\url{https://github.com/bips-hb/paper-shapley-hierarchy}}.

\subsection{Proof of Concept on a Synthetic Gaussian Data-Generating Process}\label{sec:experiments-poc}

We demonstrate the hierarchy with a controlled data-generating process (DGP) where each feature plays a known role in the predictive distribution. The forecast is multivariate Gaussian, i.e., $\bm{Y} \mid \bm{x} \sim \mathcal{N}(\bm{\mu}(\bm{x}), \bm{\Sigma}(\bm{x}))$, with $T=4$ output components and $p=4$ input features. Feature $x_1$ controls the mean forecast, $x_2$ the marginal variance with a time-increasing effect, $x_3$ both variance and correlation, and $x_4$ exclusively controls correlation without affecting any marginals. The DGP itself serves as the predictive model to be explained. Full details can be found in App.~\ref{app:experiment_1}.

For a single instance $\bm{x}$ with a strong positive correlation ($\rho \approx 0.92$, Fig.~\ref{fig:synthetic-hierarchy}), our hierarchy recovers the ground-truth feature roles: features $x_2$ and $x_3$ dominate Level~1, with the time-increasing effect of $x_2$ visible across $t = 1,\ldots, 4$, while $x_1$ receives zero attribution at every level, confirming that our hierarchy isolates predictive uncertainty from the mean prediction. Most importantly, $x_4$ shows the blind spot of component-wise attributions: it leaves all marginals invariant and therefore receives zero attribution at Level~1, yet is the dominant negative contribution at Level~2 and in the cross-component term. Across panels, the polygons connect equal quantities and provide a visual illustration of the propositions: the chain-rule linkage (Prop.~\ref{prop:chain-rule-linkage}) and the cross-component decomposition (Prop.~\ref{prop:cross-component}).

\begin{figure}[t]
    \centering
    \includegraphics[width=\linewidth]{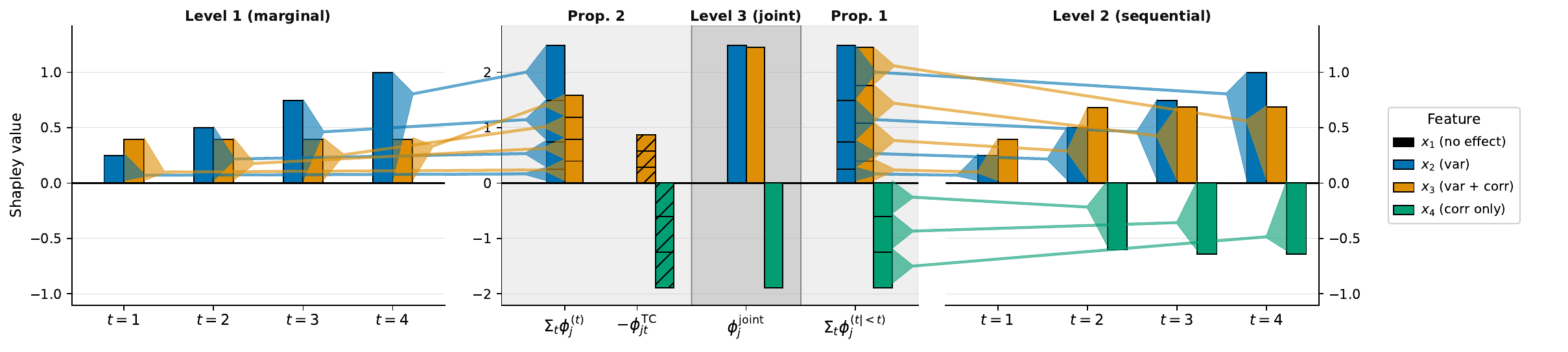}
    \caption{Entropy-Shapley hierarchy on the synthetic Gaussian DGP for an instance with high correlation. Bars connect equal quantities across panels, visualizing Prop.~\ref{prop:chain-rule-linkage} and Prop.~\ref{prop:cross-component}.}
    \label{fig:synthetic-hierarchy}
\end{figure}

\subsection{Distributional Regression on Bike Sharing}\label{sec:experiments-bikesharing}

We now move from controlled ground truth to a real distributional regression problem. We predict bike rental counts on the UCI Bike Sharing dataset~\cite{bike_sharing}, aggregated into $T=8$ two-hour blocks between 06:00 and 22:00, from $p=8$ daily features (three weather snapshots at 06:00, seven-day mean temperature, previous-day rentals, and the calendar features), using NGBoost~\cite{duan2020ngboost} with the multivariate normal as the target distributional family. This yields a feature-dependent mean and full covariance, and our hierarchy applies in closed form (Sec.~\ref{sec:method-estimation-gaussian}). Further details, comparison to standard SHAP applied to the predictive mean, and a global analysis can be found in App.~\ref{app:experiment_2}.

\begin{figure}[b]
    \centering
    \includegraphics[width=1\linewidth]{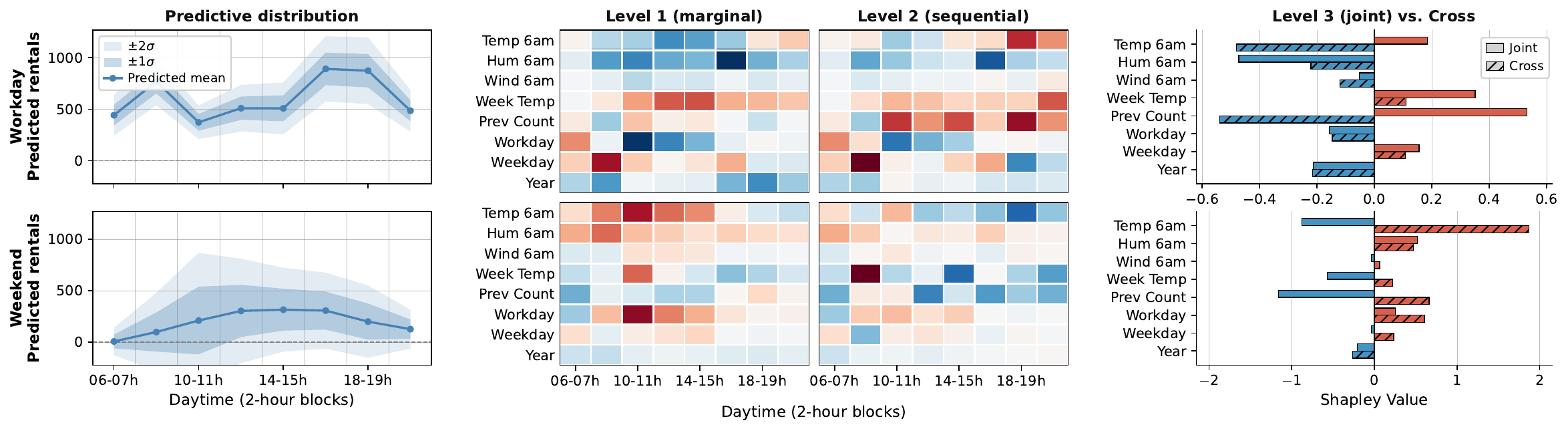}
    \caption{Entropy-Shapley hierarchy on NGBoost for two Bike Sharing instances (top: workday; bottom: weekend). Left: predicted mean and uncertainty bands. Center: Level~1 and Level~2 attributions (blue: $< 0$, red: $> 0$). Right: Level~3 and cross-component (hatched) attributions.}
    \label{fig:bikesharing-hierarchy}
\end{figure}

Figure~\ref{fig:bikesharing-hierarchy} illustrates the attributions of our hierarchy for two distinct instances\footnote{The feature values of the two instances are provided in Fig.~\ref{fig:bikesharing-instances}.}: a workday with a typical bimodal commute profile (top row, low uncertainty), and a weekend (bottom row, high uncertainty). Each level reveals a different aspect of how features drive the predictive uncertainty (left panel). At Level~1, the low morning humidity reduces the predicted uncertainty across nearly all daytime blocks for the workday instance.
Level~2 then shows the sequential view, where the same feature contributes less, because part of its information is already contained in earlier time blocks. In the right panel, Level~3 illustrates how strongly each feature drives the uncertainty of the entire forecast. On the weekend, \texttt{Temp 6am} and \texttt{Prev Count} are the most contributing features, both reducing the overall forecast uncertainty. The cross-component term then makes explicit whether a feature increases or decreases the dependence between rental blocks. For \texttt{Temp 6am} on the weekend instance, the joint contribution is negative, but the cross contribution is highly positive, showing that the morning temperature reduces the overall uncertainty while at the same time strongly increasing the dependence between the rental blocks.

\subsection{Cross-component Diagnostic across Forecaster Classes}\label{sec:experiments-deepar}

\begin{wrapfigure}{r}{0.38\textwidth}
  \begin{center}
    \includegraphics[width=0.38\textwidth]{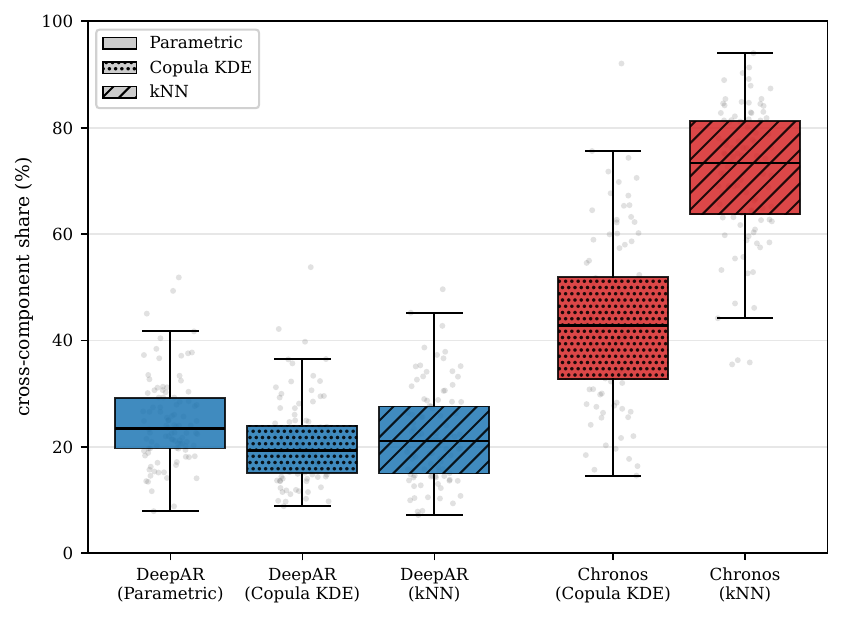}
  \end{center}
  \caption{Cross-component share on $100$ forecast origins for DeepAR and Chronos on the electricity dataset, estimated via the parametric (DeepAR only), Gaussian copula, and kNN.}
  \label{fig:model-comparison}
\end{wrapfigure}
While Sec.~\ref{sec:experiments-bikesharing} demonstrates the hierarchy on a model with closed-form predictive distribution, we now turn to forecasters whose distribution is only available through joint samples (Sec.~\ref{sec:method-estimation-samples}). On UCI Electricity~\cite{trindade2015electricity}, we compare two forecasters of different classes: DeepAR \cite{salinas2020deepar} with a Gaussian likelihood (autoregressive parametric) and Chronos \cite{ansari2024chronos} (zero-shot foundation model). For each model and $100$ forecast origins, we compute the cross-component share as
\begin{align}\label{eq:cross-share}
  \frac{\sum_j \bigl|\phi^{\operatorname{TC}}_j(\bm{x})\bigr|}
       {\sum_j \bigl|\phi^{\text{joint}}_j(\bm{x})\bigr|
        + \sum_j \bigl|\phi^{\operatorname{TC}}_j(\bm{x})\bigr|}.
\end{align}
This quantity measures the relative magnitude of the cross-component attribution mass within the combined attribution mass of the joint and cross-component terms, i.e., higher values indicate a larger relative contribution of cross-component dependence, whereas lower values indicate a larger relative contribution of the joint term.
We use the Gaussian copula and kNN estimators (Sec.~\ref{sec:method-estimation-samples}) as sample-based approaches for both models and, additionally, the approaches from Sec.~\ref{sec:method-estimation-conditional} for DeepAR using marginalization for Level~1. For DeepAR, we also conducted a more general validation of all sample-based estimators across likelihoods in App.~\ref{app:experiment_3}.

Figure~\ref{fig:model-comparison} shows a clear difference between the model classes as Chronos exhibits much more cross-component share than DeepAR, even though the forecast performance is comparably good (MAE $0.127$ for DeepAR vs.~$0.13$ for Chronos). Within each class, the two sample-based estimators on DeepAR closely match the parametric reference, likely because DeepAR's per-step Gaussian structure produces a near-Gaussian joint distribution that both estimators handle well. On Chronos, however, the copula and kNN estimations diverge substantially: Chronos models more complex and non-parametric distributions~\cite{ansari2024chronos}, so the Gaussian copula may be misspecified for the foundation model, while kNN can itself be biased on such high-complexity outputs (see Sec.~\ref{sec:method-estimation-samples}). Regardless of how this gap decomposes, the cross-component share clearly distinguishes the two model classes, providing a practical diagnostic of how much joint structure a forecaster has actually learned.


\section{Discussion and Limitations}\label{sec:discussion}

We have introduced a hierarchy of entropy-based Shapley games for explaining multivariate predictive distributions while accounting for their dependence structure, resolving a structural blind spot of standard component-wise attribution. The three levels of output resolution are linked through the chain-rule decomposition (Prop.~\ref{prop:chain-rule-linkage}) and the total-correlation characterization (Prop.~\ref{prop:cross-component}), allowing different aspects of predictive uncertainty to be attributed in a coherent way. While we have studied both the estimation of the levels (Sec.~\ref{sec:method-estimation}) and their applicability (Sec.~\ref{sec:experiments}), some open issues remain: 
\textit{(1)} The hierarchy attributes the predictive uncertainty regardless of whether it is aleatoric or epistemic in origin~\cite{Hllermeier2021}, i.e., the two can be attributed separately if the predictive distribution itself separates them, e.g., through MC dropout~\cite{Kendall2017} as in InfoSHAP~\cite{watson2023explaining}. 
\textit{(2)} Our diagnostic is only informative on models that can, in principle, represent inter-component dependence, as under conditional independence, the hierarchy collapses to a component-wise view. Moreover, the uncertainty must be driven by features rather than post-hoc trajectory sampling schemes~\cite{baron2025efficiently}, where any apparent coupling between components stems from the sampling procedure rather than the model itself. 
\textit{(3)} Sample-based estimation inherits known biases of copula and nearest-neighbour estimators on highly non-Gaussian distributions~\cite{berrett2019efficient}; as seen for Chronos (Sec.~\ref{sec:experiments-deepar}), this can produce estimator disagreement that is itself diagnostically useful. 
Future work could explore more flexible dependence structures, such as computationally feasible vine copulae, to capture stronger tail dependencies \citep{bedford2002vines, aas2009pair}.
\textit{(4)} Finally, our framework inherits two standard Shapley challenges: theoretically, the choice of background imputation method (Sec.~\ref{sec:method-estimation}) alters the interpretation of the resulting attributions~\cite{watson2023explaining}; computationally, exact evaluation scales exponentially with the number of features and the forecast horizon $T$, requiring established sampling-based approximations in higher dimensions~\cite{chen2023algorithms}.




\section*{Acknowledgements}

This work was conducted as part of the ``Uncertainty-aware feature attribution for temporal predictive models'' action, which has received funding from the European Union, via the oc3-2025-TES-01 issued and implemented by the ENFIELD project, under the grant agreement No. 101120657. Additionally, Niklas Koenen has been funded by the German Research Foundation (DFG) as part of the Research Unit ``Lifespan AI: From Longitudinal Data to Lifespan Inference in Health'' (DFG FOR 5347), Grant 459360854, and by the Emmy Noether Grant 437611051.
Finally, part of Martin's work has also been funded by the Integreat Center of Excellence -- a centre of excellence funded by Research Council of Norway, project number 332645.

\clearpage
{
\small
\bibliographystyle{elsarticle-num-names}
\bibliography{refs}
}

\clearpage
\appendix

\renewcommand{\thefigure}{\thesection.\arabic{figure}}
\renewcommand{\thetable}{\thesection.\arabic{table}}
\renewcommand{\theequation}{\thesection.\arabic{equation}}

\setcounter{figure}{0}
\setcounter{table}{0}
\setcounter{equation}{0}

\section{Additional Background}
\numberwithin{equation}{section}

\subsection{Information Theory}\label{app:information-theory}
In this appendix, we collect and explain in more detail the information-theoretic concepts used throughout the paper for our hierarchy of Shapley games, in particular the joint and conditional entropy, the (conditional) total correlation, and their behavior under chain-rule decompositions. For further details on information theory and total correlation, we refer to \citet{Cover2005} and \citet{Watanabe_TC}.

\paragraph{Joint and conditional entropy.}
The (differential) entropy $H(X)$ defined in Eq.~\eqref{eq:entropy} extends to a random vector $\bm{X}$ on a domain $\mathcal{X} \subseteq \mathbb{R}^p$ as the joint entropy
\begin{align}
    H(\bm{X}) = \mathbb{E}_{\bm{X}}\bigl[-\log\, p(\bm{X}) \bigr] = - \int_{\mathcal{X}} p(\bm{x})\, \log p(\bm{x})\, d\bm{x},
\end{align}
which depends on the full joint distribution $p(\bm{x})$. For our hierarchy, the central quantity is the conditional entropy of the multivariate output $\bm{Y}$ (on $\mathcal{Y} \subseteq \mathbb{R}^T$) given the inputs $\bm{X}$ (on $\mathcal{X} \subseteq \mathbb{R}^p$). For a specific realization $\bm{X} = \bm{x}$, the local conditional entropy
\begin{align}
    H(\bm{Y} \mid \bm{X} = \bm{x}) = - \int_{\mathcal{Y}} p(\bm{y} \mid \bm{x})\, \log \big(p(\bm{y} \mid \bm{x})\big)\, d\bm{y}
\end{align}
measures the residual uncertainty in $\bm{Y}$ for that particular instance $\bm{x}$. Averaging over the input distribution yields the global conditional entropy
\begin{align}
    H(\bm{Y} \mid \bm{X}) = \mathbb{E}_{\bm{\tilde{X}}}\bigl[H(\bm{Y} \mid \bm{X} = \bm{\tilde{X}})\bigr] &= - \int_{\mathcal{X}} p(\bm{x}) \int_{\mathcal{Y}} p(\bm{y} \mid \bm{x})\, \log \big(p(\bm{y} \mid \bm{x})\big)\, d\bm{y} d\bm{x} \notag \\
    &= -  \int_{\mathcal{X} \times \mathcal{Y}} p(\bm{y} , \bm{x})\, \log \big(p(\bm{y} \mid \bm{x})\big)\, d\bm{y} d\bm{x},
    \label{eq:cond-entropy-app}
\end{align}

which captures the expected residual uncertainty across the population. This local-vs-global distinction is used throughout Section~\ref{sec:method}: our value functions are precisely partial averages of this kind, where the expectation is taken only over the out-of-coalition features $\bm{X}_{\bar{S}}$ while $\bm{X}_S$ is held fixed at the instance value.

\paragraph{Total correlation and its conditional counterpart.}
The total correlation~\citep{Watanabe_TC} measures the overall statistical redundancy among the components of a random vector $\bm{X}$ and is defined as
\begin{align}
    \operatorname{TC}(\bm{X}) = \sum_{i=1}^p H(X_i) - H(\bm{X}),
\end{align}
i.e., the difference between the summed marginal entropies and the joint entropy. It satisfies $\operatorname{TC}(\bm{X}) \geq 0$, with equality iff the components are mutually independent. For the supervised setting, the relevant quantity is the \emph{conditional total correlation}, which measures the dependence among the multivariate outputs $\bm{Y}$ that remains \emph{after} conditioning on the covariates $\bm{X}$:
\begin{align}
    \operatorname{TC}(\bm{Y} \mid \bm{X}) = \sum_{t=1}^T H(Y_t \mid \bm{X}) - H(\bm{Y} \mid \bm{X}) = \mathbb{E}_{\bm{\tilde{X}}}\bigl[\operatorname{TC}(\bm{Y} \mid \bm{X} = \bm{\tilde{X}})\bigr],
\end{align}
following the same local-vs-global convention as the conditional entropy in Eq.~\eqref{eq:cond-entropy-app}, i.e., $\operatorname{TC}(\bm{Y}\mid \bm{X} = \bm{x}) = \sum_{t = 1}^T H(Y_t\mid \bm{X} = \bm{x}) - H(\bm{Y} \mid \bm{X} = \bm{x})$. The local quantity $\operatorname{TC}(\bm{Y} \mid \bm{X} = \bm{x})$ is the redundancy among output components for a specific instance, and Prop.~\ref{prop:cross-component} characterizes our cross-component attributions as the Shapley decomposition of the local-vs-global gap $\operatorname{TC}(\bm{Y} \mid \bm{x}) - \operatorname{TC}(\bm{Y} \mid \bm{X})$.

\paragraph{Chain rule of entropy.}
Joint entropies decompose sequentially via the chain rule of entropy,
\begin{align}
    H(\bm{X}) = \sum_{i=1}^p H(X_i \mid X_1, \dots, X_{i-1}),
\end{align}
which holds for any permutation of the indices. In our predictive setting, this extends using the conditional versions of above to $H(\bm{Y} \mid \bm{X}) = \sum_{t=1}^T H(Y_t \mid \bm{Y}_{<t}, \bm{X})$. While the joint entropy is invariant to the chosen ordering, the individual summands are not: different orderings produce different sequential decompositions, all summing to the same joint entropy. This ordering dependence is what makes our Level~2 game well-defined for ordered multivariate outputs such as time-series forecasts, where a natural ordering is given by the problem structure.

\section{Proofs of Section~\ref{sec:method}}\label{app:proofs}
\numberwithin{equation}{section}

\subsection{Proof of Prop.~\ref{prop:chain-rule-linkage} (Chain-rule linkage)}

\begin{proof} \label{proof:chain-rule-linkage}
The value function of the joint entropy game (Eq.~\eqref{eq:level3}) decomposes into the sum over time of the sequential entropy game value functions (Eq.~\eqref{eq:level2}):

 \begin{equation} \label{eq:chain-rule-linkage-proof}
\begin{split}
 v_{H}^\text{joint}(S, \bm{x}) &=\mathbb{E}_{\bm{X}_{\bar{S}}} \biggl[H\bigl(\bm{Y} \bigm|  \bm{X} = (\bm{x}_S, \bm{X}_{\bar{S}})\bigr)\biggr]\\
 &=\mathbb{E}_{\bm{X}_{\bar{S}}} \biggl[\sum_{t=1}^{T} H\bigl(Y_t \bigm| \bm{Y}_{<t}, \bm{X} = (\bm{x}_S, \bm{X}_{\bar{S}})\bigr)\biggr] \\
 &=\sum_{t=1}^{T}  \mathbb{E}_{\bm{X}_{\bar{S}}} \biggl[H\bigl(Y_t \bigm| \bm{Y}_{<t}, \bm{X} = (\bm{x}_S, \bm{X}_{\bar{S}})\bigr)\biggr]
 \\
 &=\sum_{t=1}^{T} v_{H}^{(t \mid <t)}(S, \bm{x}) 
\end{split}
\end{equation}
where the second line applies the chain rule of (conditional) entropy, while the third line uses the linearity of expectation.

The decomposition of the corresponding Shapley values $\phi^{\text{joint}}(j, \bm{x}) = \sum_{t=1}^{T} \phi^{(t \mid <t)}(j, \bm{x})$
follows directly from Eq.~\eqref{eq:chain-rule-linkage-proof} and the linearity axiom of Shapley values.

Finally, if the output components $Y_1, \ldots, Y_T$ are conditionally independent given $\bm{X}$, then $H(Y_t \mid \bm{Y}_{<t}, \bm{X}) = H(Y_t \mid \bm{X})$ for any $\bm{X}$. Consequently, the Level~2 and Level~1 value functions coincide for all coalitions, yielding identical Shapley values $\phi^{(t \mid <t)}(j, \bm{x}) = \phi^{(t)}(j, \bm{x})$ and reducing the joint attribution to $\phi^{\text{joint}}(j, \bm{x}) = \sum_{t=1}^T \phi^{(t)}(j, \bm{x})$.
\end{proof}

\subsection{Proof of Prop.~\ref{prop:cross-component} (Cross-component decomposition)}

\begin{proof} \label{proof:cross-component-decomposition}
The value function of the total correlation game mirrors the decomposition of the total correlation into marginal and joint conditional entropies:

\begin{equation} \label{eq:tc-value-func}
\begin{split}
v_{\operatorname{TC}}(S, \bm{x}) 
&= \mathbb{E}_{\bm{X}_{\bar{S}}}[\operatorname{TC}(\bm{Y} \mid \bm{X}) \mid \bm{X} = (\bm{x}_S, \bm{X}_{\bar{S}})] \\
&= \mathbb{E}_{\bm{X}_{\bar{S}}}[\sum_{t=1}^{T} H(Y_t  \mid \bm{X} = (\bm{x}_S, \bm{X}_{\bar{S}})) - H(\bm{Y}  \mid \bm{X} = (\bm{x}_S, \bm{X}_{\bar{S}}))] \\
&= \sum_{t=1}^{T}\mathbb{E}_{\bm{X}_{\bar{S}}}[H(Y_t  \mid \bm{X} = (\bm{x}_S, \bm{X}_{\bar{S}}))] - \mathbb{E}_{\bm{X}_{\bar{S}}}[H(\bm{Y}  \mid \bm{X} = (\bm{x}_S, \bm{X}_{\bar{S}}))] \\
&= \sum_{t=1}^{T} v_{H}^{(t)}(S, \bm{x}) - v_{H}^\text{joint}(S, \bm{x})
\end{split}
\end{equation}
where in the second line we used the definition of total correlation and in the third line the linearity of the expected value operator.
From Eq.~\eqref{eq:tc-value-func} and the linearity of the Shapley values in the value function follows that $\phi^{\operatorname{TC}}(j, \bm{x}) = \sum_{t} \phi^{(t)}(j, \bm{x}) - \phi^{\text{joint}}(j, \bm{x})$, where $\phi^{\operatorname{TC}}(j, \bm{x})$ are the Shapley values of the total correlation game $v_{\operatorname{TC}}(S, \bm{x})$.


\end{proof}

\subsection{Local--Global conditional entropy decomposition of Shapley values}\label{app:local-global-entropy-decomposition}

\begin{corollary}\label{cor:local-global-entropy-decomposition}
For the Level~1--3 games (Eqs.~\eqref{eq:level1}--\eqref{eq:level3}), the sum of the Shapley values equals the difference between the local and global conditional entropies:
\begin{equation} \label{eq:local-global-entropy-decomposition}
\begin{split}
\sum_{j=1}^p \phi^{(t)}(j, \bm{x}) &= H(Y_t \mid \bm{x}) - H(Y_t \mid \bm{X}) \qquad \text{(Level~1)}\\
\sum_{j=1}^p \phi^{(t \mid <t)}(j, \bm{x}) &= H(Y_t \mid \bm{Y}_{<t}, \bm{x}) - H(Y_t \mid \bm{Y}_{<t}, \bm{X})  \qquad \text{(Level~2)}\\
\sum_{j=1}^p \phi^{\text{joint}}(j, \bm{x}) &=H(\bm{Y} \mid \bm{x}) - H(\bm{Y} \mid \bm{X}) \qquad \text{(Level~3)}\\
\end{split}
\end{equation}
Analogously, for the TC game (Eq. ~\eqref{eq:tc-value-func}), the sum of the Shapley values matches the local--global total correlation gap: 
\begin{equation} \label{eq:local-global-entropy-decomposition-tc}
\begin{split}
\sum_{j=1}^p \phi^{\operatorname{TC}}(j, \bm{x}) =&\operatorname{TC}(\bm{Y} \mid \bm{x}) - \operatorname{TC}(\bm{Y} \mid \bm{X}).
\end{split}
\end{equation}
\end{corollary}
\begin{proof}
For each game: By the efficiency property of Shapley values, the sum of Shapley values $\sum_{j=1}^p \phi(j, \bm{x})$ equals the difference between the value function of the grand coalition ($S=[p]$) and the value function of the empty coalition ($S=\emptyset$). For our games, the grand coalition evaluates the local conditional entropies at the specific instance $\bm{x}$, while the empty coalition evaluates the global conditional entropies marginalized over the entire input space.
\end{proof}



\section{Trajectory-Based Gaussian Copula Estimation}\label{app:copula}
\numberwithin{equation}{section}

In Section \ref{sec:method-estimation-samples}, we introduced a semiparametric Gaussian copula approach to efficiently estimate the local entropy terms required for our hierarchy of Shapley games. This method leverages simulated trajectories to capture complex, non-Gaussian marginal distributions while utilizing the computational tractability of Gaussian dependencies. 

Below, we detail the complete, step-by-step procedure for estimating the value functions of a specific feature coalition $S$ for a given instance $\bm{x}$. The procedure evaluates the inner conditional entropies for a single completed input $\tilde{\bm{x}} = \tilde{\bm{x}}^{(k)}$ (generated via the chosen Shapley imputation method, as described in Sec.~\ref{sec:method-estimation}) and then averages over $K$ such completed inputs to evaluate the outer expectation.

\begin{enumerate}
    \item \textbf{Trajectory Simulation:} For a completed input $\tilde{\bm{x}}$, simulate $N$ joint trajectories from the predictive model:
    $$ \bm{y}^{(1)}, \dots, \bm{y}^{(N)} \sim p(\bm{Y} \mid \tilde{\bm{x}}) $$
    
    \item \textbf{Marginal Density Estimation:} For each output component $t \in \{1, \dots, T\}$, estimate the 1D marginal density $\hat{f}_t(\cdot \mid \tilde{\bm{x}})$ and the cumulative distribution function $\hat{F}_t(\cdot \mid \tilde{\bm{x}})$ using kernel density estimation (KDE) on the scalar samples $\{y_t^{(i)}\}_{i=1}^N$.
    
    \item \textbf{Level 1 (Marginal Entropy):} Estimate the Level 1 entropy for each component $t$ directly from the fitted marginal densities e.g., via the leave-one-out plug-in estimator:
    \begin{align}
        \widehat{H}(Y_t \mid \tilde{\bm{x}}) = -\frac{1}{N}\sum_{i=1}^N \log \hat{f}_{t,-i}(y_t^{(i)} \mid \tilde{\bm{x}})
        \label{eq:app-copula-l1}
    \end{align}
    where $\hat{f}_{t,-i}$ denotes the KDE fitted without the $i$-th sample to avoid overfitting.
    
    \item \textbf{Gaussianization \& Correlation:} Map the samples to the uniform domain using the empirical CDFs, $u_t^{(i)} = \hat{F}_t(y_t^{(i)} \mid \tilde{\bm{x}})$, and subsequently to the standard normal domain via the probit function, $z_t^{(i)} = \Phi^{-1}(u_t^{(i)})$. Compute the $T \times T$ empirical correlation matrix $\hat{\bm{R}}(\tilde{\bm{x}})$ from these latent Gaussian vectors $\bm{z}^{(i)}$.
    
    \item \textbf{Level 2 (Sequential Entropy):} For each step $t > 1$, calculate the conditional variance of the latent Gaussian variable $Z_t$ given the preceding variables $Z_{<t}$ using the Schur complement:
    $$ \hat{\sigma}^2_{t \mid <t, z}(\tilde{\bm{x}}) = 1 - \hat{\bm{R}}_{t,<t} \hat{\bm{R}}_{<t,<t}^{-1} \hat{\bm{R}}_{<t,t} $$
    Use this variance to compute the Level 2 sequential entropy:
    \begin{align}
        \widehat{H}(Y_t \mid \bm{Y}_{<t}, \tilde{\bm{x}}) = \widehat{H}(Y_t \mid \tilde{\bm{x}}) + \frac{1}{2}\log\!\bigl(\hat{\sigma}^2_{t \mid <t, z}(\tilde{\bm{x}})\bigr)
        \label{eq:app-copula-l2}
    \end{align}
    
    \item \textbf{Level 3 (Joint Entropy):} Assemble the total joint entropy by summing the Level 2 components (or equivalently, Level 1 plus the copula entropy), effectively applying the chain rule from Prop.~\ref{prop:chain-rule-linkage}:
    $$ \widehat{H}(\bm{Y} \mid \tilde{\bm{x}}) = \sum_{t=1}^T \widehat{H}(Y_t \mid \bm{Y}_{<t}, \tilde{\bm{x}}) $$

    \item \textbf{Outer Expectation:} Repeat Steps 1--6 for $K$ different completed inputs $\tilde{\bm{x}} = \tilde{\bm{x}}^{(k)}$ generated according to the chosen imputation strategy (e.g., marginal, conditional, or baseline). Average the resulting entropy estimates to yield the final coalition value functions $v_H^{(t)}(S, \bm{x})$, $v_H^{(t \mid <t)}(S, \bm{x})$, and $v_H^{\text{joint}}(S, \bm{x})$.    
\end{enumerate}

\section{Details and Additional Results on the Experiments}\label{app:exp-details}

In the following sections, we describe the technical details and additional results for all experiments covered in Section~\ref{sec:experiments}. 

\textbf{General Experimental Setup.} Throughout all experiments, we compute the exact Shapley values by evaluating the respective value functions across all $2^p$ possible feature coalitions, without relying on coalition-sampling approximations. Furthermore, to evaluate the outer expectation over the out-of-coalition features (i.e., the Shapley imputation step described in Sec.~\ref{sec:method-estimation}), we consistently employ marginal imputation using background samples drawn from the respective training datasets \citep{lundberg2017}.
Finally, the code for reproducing the results can be found on \href{https://github.com/nkoenen/paper-shapley-hierarchy/}{GitHub}.

\subsection{Proof of Concept on Synthetic Gaussian DGP}\label{app:experiment_1}

For this synthetic experiment, the data is generated using a multivariate Gaussian distribution $\bm{Y} \mid \bm{x} \sim \mathcal{N}(\bm{\mu}(\bm{x}), \bm{\Sigma}(\bm{x}))$ with $T = 4$ output components and $p = 4$ input features drawn independently from a standard normal. Each feature has a distinct ground-truth effect on the parameters of the predictive distribution. The mean vector $\bm{\mu}(\bm{x})$ depends only on the first feature $x_1$ via
\begin{align*}
    \mu_t(\bm{x}) = \beta\, x_1\, \sin(2\pi t / T)\quad \text{with} \quad \beta = 1.
\end{align*}
The covariance is constructed as $\bm{\Sigma}(\bm{x}) = D(\bm{x})R(\bm{x})D(\bm{x})$, where $D(\bm{x})$ is the diagonal of marginal standard deviations
\begin{align*}
    \sigma_t^2(\bm x) = \exp \bigl(\gamma_0\, x_2 \, t + \gamma_1\, x_3 \bigr)\quad \text{with}\quad \gamma_0 = 0.5, \gamma_1 =  0.8,
\end{align*}
and $R(\bm{x})$ is a smoothly decreasing Toeplitz-style correlation matrix driven by
\begin{align*}
    r(\bm x) = \tanh \bigl(\delta_0 x_3 + \delta_1 x_4\bigr), \quad \delta_0 = -0.6, \delta_1 = 1.25, \qquad R_{ij}(\bm x) = \mathrm{sign}(r(\bm{x}))^{|i-j|}\, |r(\bm{x})|^{|i-j|^{1.5}},
\end{align*}
where the correlation decays faster than linearly with the distance $|i - j|$, so that neighbouring components remain strongly correlated but distant ones are nearly uncorrelated. Through this design, each feature has pre-defined effects on the predictive distribution: $x_1$ controls the mean prediction only and should be invisible for uncertainty attribution, $x_2$ has an increasing effect on the marginal variance for higher component indices, $x_3$ controls both variance and correlation, whereas $x_4$ only controls the correlation. 

For the calculation of the Shapley values for the different levels, we use the analytical formula for the entropies (see Sec.~\ref{sec:method-estimation-gaussian}) and use $10,000$ background samples for the Monte-Carlo integration in order to have stable results. The explanations in Figure~\ref{fig:synthetic-hierarchy} are based on the instance $\bm{x} = (1,1,1,1.75)$, i.e., an instance with very high correlations ($\rho \approx 0.92$) for demonstration purposes.

\paragraph{Visualizing the hierarchy on the covariance matrix $\bm{\Sigma}$.} For the multivariate Gaussian case (see Sec.~\ref{sec:method-estimation-gaussian}), the three levels of the hierarchy operate on different parts of the covariance matrix $\bm{\Sigma}(\bm{x})$, which provides a geometric intuition for what each level captures. Figure~\ref{fig:cov-decomposition} illustrates this for $T=4$: Level~1 depends only on the diagonal entries $\Sigma_{tt}$, Level~2 depends on the leading principal submatrices through Schur complements $\Sigma_{t \mid <t}$, and Level~3 depends on the full covariance matrix via its log-determinant. The chain-rule decomposition (Prop.~\ref{prop:chain-rule-linkage}) is reflected in this structure: the joint log-determinant decomposes into the sum of conditional log-variances, which collapses to the sum of marginal log-variances exactly when the components are independent, i.e., when $\bm{\Sigma}$ is diagonal.

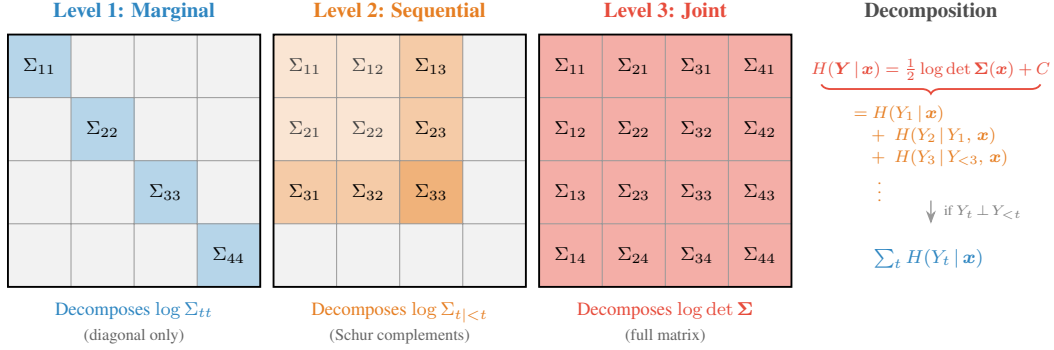
\begin{figure}[!h]
    \centering
    \resizebox{\textwidth}{!}{\begin{tikzpicture}[
    >=Stealth,
    cell/.style={minimum size=0.9cm, inner sep=0pt, outer sep=0pt},
    ]
    
    \def\n{4}
    \def\s{0.95}
    
    \begin{scope}
        
        \fill[cGray!25] (0,0) rectangle ({\n*\s},{\n*\s});
        
        \foreach \i in {0,...,3}{
            \fill[cL1!35] (\i*\s,{(\n-1-\i)*\s}) rectangle ++(\s,\s);
        }
        \foreach \i in {1,...,4} {
            \node[cell, font=\footnotesize] at ({(\i-0.5)*\s}, {(5-\i-0.5)*\s}) 
                {$\Sigma_{\i\i}$};
        }
        
        \draw[thin,black!40] (0,0) grid[step=\s] ({\n*\s},{\n*\s});
        \draw[thick] (0,0) rectangle ({\n*\s},{\n*\s});
        
        \node[font=\small\bfseries, cL1, anchor=north] at ({\n*\s/2}, {\n*\s + 0.6}) 
        {Level 1: Marginal};
        
        \node[font=\footnotesize, cL1, anchor=north, align=center] at ({\n*\s/2}, {-0.1}) 
        {Decomposes $\log \Sigma_{tt}$\\[1.5pt]
            \textcolor{black!60}{\scriptsize (diagonal only)}};
        
    \end{scope}
    
    \begin{scope}[xshift=4cm]
        
        \fill[cGray!25] (0,0) rectangle ({\n*\s},{\n*\s});
        
        \fill[cL2!20] (0,{2*\s}) rectangle ({2*\s},{4*\s});
        \fill[cL2!40] ({2*\s},{2*\s}) rectangle ({3*\s},{4*\s});
        \fill[cL2!40] (0,{1*\s}) rectangle ({2*\s},{2*\s});
        \fill[cL2!60] ({2*\s},{1*\s}) rectangle ({3*\s},{2*\s});
        
        \draw[thin,black!40] (0,0) grid[step=\s] ({\n*\s},{\n*\s});
        \draw[thick] (0,0) rectangle ({\n*\s},{\n*\s});
        
        \node[font=\small\bfseries,cL2,anchor=north]
        at ({\n*\s/2},{\n*\s+0.6}) {Level 2: Sequential};
                    
        \node[cell, font=\footnotesize, black!75] at ({(1-0.5)*\s}, {(5-1-0.5)*\s}) {$\Sigma_{11}$};
        \node[cell, font=\footnotesize, black!75] at ({(2-0.5)*\s}, {(5-2-0.5)*\s}) {$\Sigma_{22}$};
        \node[cell, font=\footnotesize, black!75] at ({(1-0.5)*\s}, {(5-2-0.5)*\s}) {$\Sigma_{21}$};
        \node[cell, font=\footnotesize, black!75] at ({(2-0.5)*\s}, {(5-1-0.5)*\s}) {$\Sigma_{12}$};
        
        \foreach \i in {1,2} {
            \node[cell, font=\footnotesize, black] at ({(\i-0.5)*\s}, {(2-0.5)*\s}) {$\Sigma_{3\i}$};
            \node[cell, font=\footnotesize, black] at ({(3-0.5)*\s}, {(5-\i-0.5)*\s}) {$\Sigma_{\i3}$};
        }
        \node[cell, font=\footnotesize, black] at ({(3-0.5)*\s}, {(5-3-0.5)*\s}) {$\Sigma_{33}$};
        
        \node[font=\footnotesize, cL2, anchor=north, align=center] at ({\n*\s/2}, {-0.1}) 
        {Decomposes $\log \Sigma_{t|<t}$\\[1.5pt]
            \textcolor{black!60}{\scriptsize (Schur complements)}};
    \end{scope}

    \begin{scope}[xshift=8cm]
        
        \fill[cL3!45] (0,0) rectangle ({\n*\s},{\n*\s});
        
        \draw[thin,black!40] (0,0) grid[step=\s] ({\n*\s},{\n*\s});
        \draw[thick] (0,0) rectangle ({\n*\s},{\n*\s});
        
        \node[font=\small\bfseries,cL3,anchor=north]
        at ({\n*\s/2},{\n*\s+0.6}) {Level 3: Joint};
        
        \foreach \i in {1,...,4} {
            \foreach \j in {1,...,4} {
                \pgfmathtruncatemacro{\row}{5-\j}
                \ifnum\i=\row
                \node[cell, font=\footnotesize] 
                at ({(\i-0.5)*\s}, {(\j-0.5)*\s}) 
                {$\Sigma_{\i\i}$};
                \else
                \pgfmathtruncatemacro{\ri}{\i}
                \pgfmathtruncatemacro{\rj}{\row}
                \node[cell, font=\footnotesize, black] 
                at ({(\i-0.5)*\s}, {(\j-0.5)*\s}) 
                {$\Sigma_{\ri\rj}$};
                \fi
            }
        }
        
        \node[font=\footnotesize, cL3, anchor=north, align=center] at ({\n*\s/2}, {-0.1}) 
        {Decomposes $\log\det\bm{\Sigma}$\\[1.5pt]
            \textcolor{black!60}{\scriptsize (full matrix)}};
        
    \end{scope}
    
    \begin{scope}[xshift=12cm]
        
        
        \node[font=\small\bfseries, black!70, anchor=north] 
        at ({\n*\s/2}, {\n*\s + 0.6}) {Decomposition};
        
        \node[font=\footnotesize, cL3, anchor=south, scale = 0.9] 
        (L3eq) at ({\n*\s/2}, {\n*\s - 0.8}) {
            $H(\bm{Y} \mid \bm{x}) = \tfrac{1}{2}\log \det \bm{\Sigma}(\bm{x}) + C$
        };
        
        \draw[decorate, decoration={brace, amplitude=4pt, mirror}, 
        thick, cL3]
        (0.25, {\n*\s - 0.75}) -- ({\n*\s - 0.25}, {\n*\s - 0.75});
        
        \node[font=\footnotesize, cL2, anchor=north, align=left, scale = 0.9] 
        (L2eq) at ({\n*\s/2}, {\n*\s - 0.95}) {
            $= H(Y_1 \mid \bm{x})$\\[1pt]
            $\hphantom{=} + \; H(Y_2 \mid Y_1,\, \bm{x})$\\[1pt]
            $\hphantom{=} + \; H(Y_3 \mid Y_{<3},\, \bm{x})$\\[1pt]
            $\hphantom{=} \;\;\vdots$
        };
        
        \draw[->, semithick, black!40] 
        ({\n*\s/2}, {\n*\s - 2.5}) -- ++(0, -0.35)
        node[midway, right=2pt, font=\tiny, black!50, align=left] 
        {if $Y_t \!\perp\! Y_{<t}$};
        
        \node[font=\footnotesize, cL1, anchor=north, align=left] 
        (L1eq) at ({\n*\s/2}, {\n*\s - 3.1}) {
            $\sum_t H(Y_t \mid \bm{x})$
        };
        
    \end{scope}
    
\end{tikzpicture}}
    \caption{Geometric decomposition of the multivariate Gaussian entropy hierarchy on the covariance matrix $\bm{\Sigma}(\bm{x})$. Level~1 (left) depends only on the diagonal entries $\Sigma_{tt}$, Level~2 (centre) on the leading principal submatrices via Schur complements $\Sigma_{t \mid <t}$, and Level~3 (right) on the full matrix via its log-determinant. The right panel shows how the chain rule of entropy decomposes the joint entropy into a sum of conditional entropies, collapsing to a sum of marginals under conditional independence.}
    \label{fig:cov-decomposition}
\end{figure}

\subsection{Distributional Regression on Bike Sharing}\label{app:experiment_2}

\textbf{Dataset.} We use the UCI Bike Sharing hourly dataset \cite{bike_sharing} and restrict the analysis to the demand window 06:00--22:00 of each day. The hourly rental counts are aggregated into $T = 8$ two-hour blocks $\bm{Y} = (Y_1, \ldots, Y_8)$ that serve as the multivariate target. We use $p = 8$ features as input, summarized in Table~\ref{tab:bikesharing-features}. The continuous weather features are observed at 06:00 of the forecast day, and \texttt{Week Temp} and \texttt{Prev Count} are lagged by one day. This setting ensures that no information from the target window 06:00--22:00 leaks into the features.

\begin{table}[h]
    \centering
    \caption{Feature set used for the Bike Sharing experiment ($p = 8$).}
    \begin{tabular}{ll}
        \toprule
        Feature & Description \\
        \midrule
        \texttt{Temp 6am}   & 06:00 normalized temperature \\
        \texttt{Hum 6am}    & 06:00 normalized humidity \\
        \texttt{Wind 6am}   & 06:00 normalized wind speed \\
        \texttt{Week Temp}  & 7-day rolling mean temperature, lagged by 1~day \\
        \texttt{Prev Count} & yesterday's mean rentals per hour (scaled by $1/100$) \\
        \texttt{Workday}    & binary indicator (1 if working day) \\
        \texttt{Weekday}    & day of the week (0 = Sun, $\ldots$, 6 = Sat) \\
        \texttt{Year}       & calendar year (0 = 2011, 1 = 2012) \\
        \bottomrule
    \end{tabular}
    \label{tab:bikesharing-features}
\end{table}

\textbf{Model.} We split the data into 80\% training and 20\% test instances after a random shuffle. On the training set, we fit an NGBoost model \cite{duan2020ngboost} with a multivariate Gaussian output head $\mathcal{N}(\bm{\mu}(\bm{x}), \bm{\Sigma}(\bm{x}))$ over the $T = 8$ time blocks. The model parameterizes the mean vector and a Cholesky factor of the covariance matrix. We use the default base learner (a depth-2 decision tree) and train for up to 1500 boosting rounds with learning rate $0.01$, minibatch fraction $0.8$, and early stopping after 200 rounds without improvement on the held-out test set.

\textbf{Shapley computation.} Since NGBoost outputs a parametric multivariate Gaussian, all three entropy levels are evaluated analytically using the closed-form expressions from Sec.~\ref{sec:method-estimation-gaussian}. For the local analysis on the two instances shown in Fig.~\ref{fig:bikesharing-instances}, we use the full training set as background to approximate the expectation over the off-coalition features.

\begin{figure}[!h]
    \centering
    \includegraphics[width=1.0\linewidth]{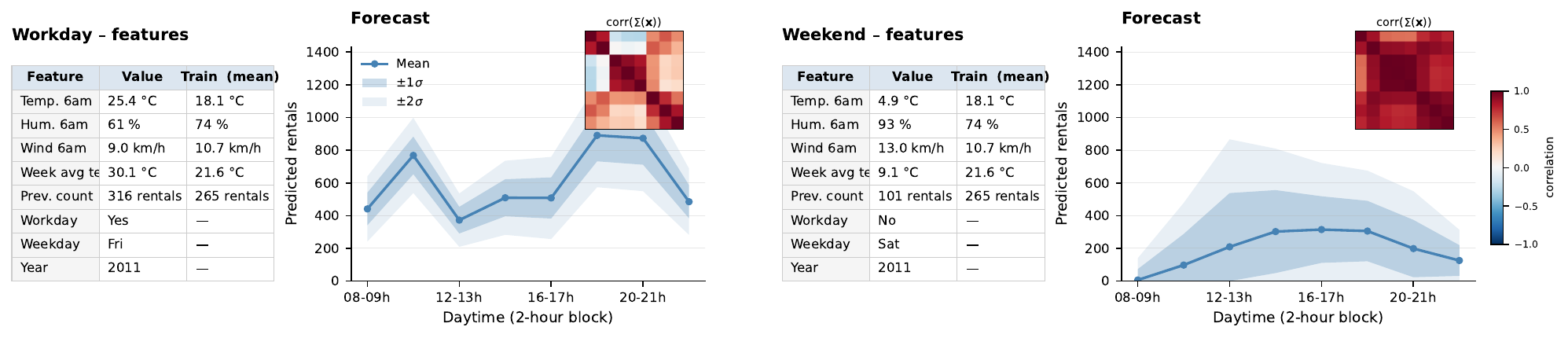}
    \caption{Summary of the feature values (including the population means) and the predictive distribution of the two instances to be explained. The heatmaps show the correlation of the predictive distribution, i.e.,  of $\bm{\Sigma}(\bm{x})$ for the two instances.}
    \label{fig:bikesharing-instances}
\end{figure}

\subsubsection{Additional Result: Comparison to SHAP}\label{app:comparison-shap}

The hierarchy in Figure~\ref{fig:appendix-shap-vs-es} captures a fundamentally different quantity than standard SHAP, since SHAP (or TreeSHAP\cite{treeshap} for NGBoost) decomposes the \emph{point prediction}, i.e., each $\phi^{\text{SHAP}}_{j,t}$ quantifies by how many rentals feature $j$ shifts the predicted mean $\hat\mu_t(\bm{x})$ at time point $t$. The games of our hierarchy, in contrast, decompose the \emph{predictive uncertainty}, i.e., $\phi_H^{(t)}$, $\phi_H^{(t|<t)}$, and $\phi_H^{\text{joint}}$ quantify how much each feature contributes to the marginal, sequentially conditional, and joint entropy of the forecast distribution, measured in nats. These are different quantities of $p(\bm{y} \mid \bm{x})$, and there is no reason to expect them to agree, not even on the ranking of features. Figure~\ref{fig:appendix-shap-vs-es} makes this distinction concrete and shows different patterns for the different levels compared to SHAP.

\begin{figure}[!h]
    \centering
    \includegraphics[width=1.0\linewidth]{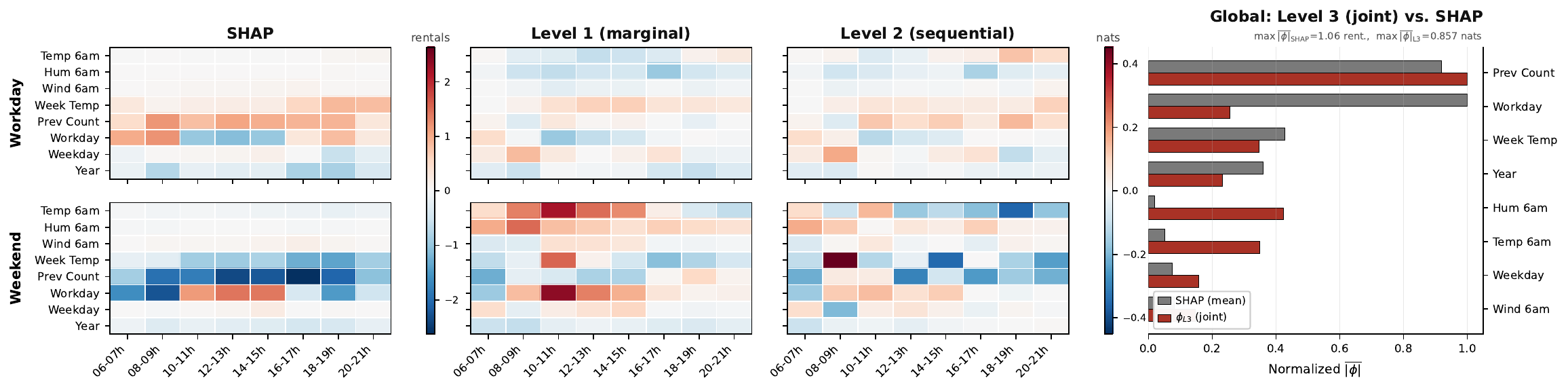}
    \caption{Local SHAP vs.\ Entropy-Shapley attributions for the workday and weekend test instances (left, centre) and global aggregation over the full test set (right). SHAP attributes the predicted mean $\hat\mu_t(\bm{x})$ in rentals, while Level~1 ($\phi_H^{(t)}$) and Level~2 ($\phi_H^{(t|<t)}$) attribute the marginal and sequentially conditional predictive entropy in nats, respectively. The right panel compares the mean of $|\phi_H^{\text{joint}}|$ across the test set with the corresponding mean absolute SHAP values, each min-max normalized within its method. This highlights how feature effects on the mean and uncertainty are naturally decoupled: for instance, \texttt{Workday} strongly drives point prediction but not the joint (daily) uncertainty, whereas the opposite is true for \texttt{Hum 6am} and \texttt{Temp 6am}.}
    \label{fig:appendix-shap-vs-es}
\end{figure}

\subsection{Cross-component Diagnostic across Forecaster Classes}\label{app:experiment_3}

This appendix summarizes the setup and additional analyses for the sample-based estimation and cross-component diagnostic experiment of Sec.~\ref{sec:experiments-deepar}.

\textbf{Dataset.} We use the UCI Electricity hourly load dataset~\cite{trindade2015electricity}, restricted to the first $100$ series in order to keep the training time feasible. The hourly electricity load is aggregated into $T = 12$ two-hour blocks per day, i.e., blocks of width 2h with a prediction window covering a whole day. For each series, we scaled the values by their per-series mean, so the analyzed target is on a relative-to-mean scale. The multivariate forecast is of dimension $T = 12$ (i.e.~24h), the history for context is $84$ blocks (i.e.~one week), and we hold out the last $360$ blocks per series as a validation period from which the test origins are drawn. In addition, we add calendar features (hour, weekday, and month) as exogenous variables.

\textbf{Models.} We train DeepAR~\cite{salinas2020deepar} from scratch with three different one-step output distributions: Gaussian, Student-$t$, and log-normal. We used a default architecture with $2$ RNN layers, a hidden size of $64$ and trained it for up to $50$ epochs with Adam (learning rate $10^{-3}$,
batch size $256$) and early stopping on the validation set. On the other hand, we use \texttt{Chronos-T5-base}\footnote{See Hugging Face: \url{https://huggingface.co/amazon/chronos-t5-base}}~\cite{ansari2024chronos}
zero-shot, i.e., without any fine-tuning. The model tokenizes the history context with its quantization grid (4096 token bins) and decodes $T = 12$ tokens autoregressively.

\paragraph{Shapley setup (Sec.~\ref{sec:experiments-deepar}).} The historical context plus three calendar features yields $P_{\text{feat}} = 87$ scalar inputs per origin. To keep the Shapley computation tractable, we group these into $p = 6$ players: one player per calendar feature (hour, weekday, month), and three lookback-window players capturing the recent past (last 24 hours), the middle past (the preceding three days), and the older past (the remainder of the week-long context). For each model and forecast origin ($100$ in total), we enumerate all $2^p = 64$ coalitions and impute features outside the coalition by drawing $K = 25$ background contexts from the training period, ensuring that no validation information leaks into the value function. For each coalition, we sample $N = 1000$ joint trajectories of length $T$ and evaluate the three entropy levels with the sample-based estimators of Sec.~\ref{sec:method-estimation-samples} (Gaussian copula using KDE marginals and kNN with $k=5$). Shapley values are obtained exactly from the $2^p$ value-function evaluations.

\subsubsection{Additional Results: Validation against the Model-based Reference}\label{app:deepar-validation}

We train DeepAR~\cite{salinas2020deepar} on the dataset using Gaussian, Student-$t$, and log-normal likelihoods for the one-step conditional distributions (Sec.~\ref{sec:method-estimation-conditional}). DeepAR's autoregressive factorization provides the conditional densities in closed form, which can be marginalized over trajectory samples of $\bm{Y}_{<t}$ to obtain the Level~2 reference $v^{(t|<t)}_H(S,\bm{x})$, with Level~3 following from Prop.~\ref{prop:chain-rule-linkage}. This reference is itself an estimator in the strict sense, but uses model-internal closed-form conditionals rather than relying on entropy estimation from samples. It therefore serves as the strongest available baseline. We compare the sample-based estimators from Sec.~\ref{sec:method-estimation-samples} against this reference, using a parametric Gaussian fit (PG), the Gaussian copula with KDE marginals (GC-KDE), and kNN. In addition, we include a Gaussian copula variant (GC-Param) whose marginals are obtained from the marginalized model conditionals, which isolates the copula component of the model bias. For this validation, we use the setup described above, but vary the trajectory budget $N \in \{50, 100, 250, 500, 1000, 2000\}$ to study the convergence behavior of each estimator. Figure~\ref{fig:deepar_val} reports the mean absolute error on the value function (top row) and on the resulting Shapley values (bottom row) as a function of $N$. The marginal imputation uses $K=25$ random background origins, fixed across all settings to eliminate Monte-Carlo integration variance from the comparison. Error bars are computed over $100$ test instances.

For the Gaussian likelihood, PG and the copula-based approaches are all accurate, while kNN improves more slowly with $N$, reflecting the cost of imposing no parametric structure. The small gap between GC-Param and GC-KDE quantifies the additional cost of estimating marginals from samples. The non-Gaussian likelihoods expose the effect of misspecification. For Student-$t$, PG's error grows with $N$, indicating convergence to a misspecified Gaussian target: the downward bias of the empirical log-determinant masks the misspecification at small $N$ and reveals it as $N$ grows. The copula-based estimators substantially reduce this effect, with the residual GC-Param error isolating the copula component, although a Gaussian copula can itself be misspecified under strong tail dependence. For log-normal, errors are smaller throughout, indicating that the severity of misspecification depends on the predictive distribution. The bottom row shows how these errors propagate to the attributions: common offsets across coalitions partly cancel in marginal contributions, while coalition-specific errors remain.

\begin{figure}[!h]
    \centering
     \includegraphics[width=1\textwidth]{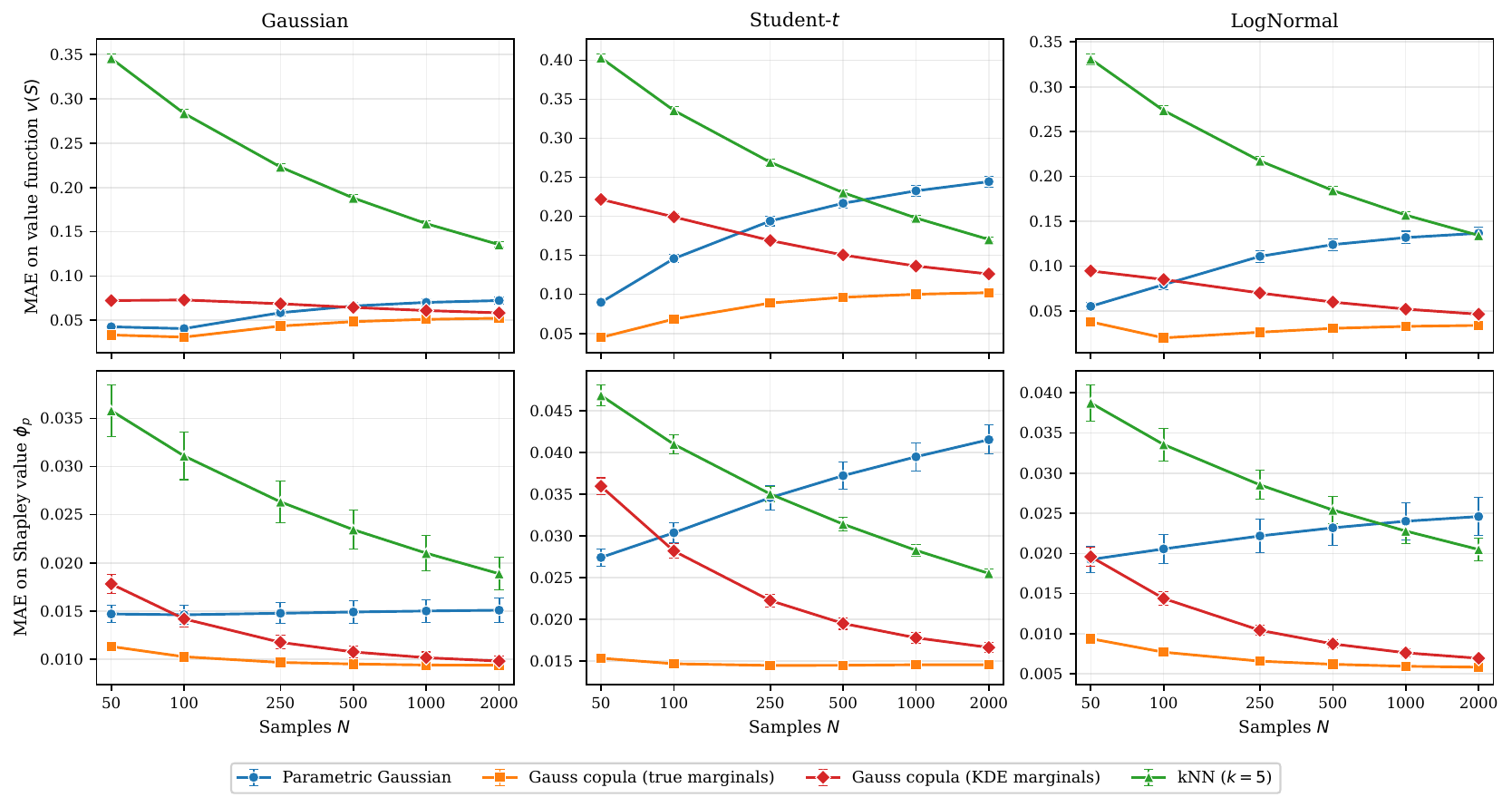}
    \caption{Sample-based estimation against the parametric model-based Level~2 reference $v^{(t|<t)}_H(S,\bm{x})$, for DeepAR with Gaussian, Student-$t$, and log-normal likelihoods. \textbf{Top:} MAE on the value function. \textbf{Bottom:} MAE on the resulting Shapley values.}
        \label{fig:deepar_val}
\end{figure}

 \section{Computational Details}\label{app:comp-details}

Experiments were conducted using a 64-bit Linux platform running Ubuntu 22.04 LTS with two AMD EPYC Genoa 9534 64-Core processors (128 cores, 256 threads total), 1.5 terabytes of RAM, and eight NVIDIA RTX 6000 Ada Generation GPUs (each with 48 GB memory). Each individual run used a single GPU. Only DeepAR training and Chronos inference require the GPU, whereas analytical and sample-based Shapley computations, including NGBoost training, run on CPU with process-level parallelism via \texttt{joblib}. Additionally, the memory-expensive notebooks implement a chunk-wise calculation to prevent out-of-memory failures. In this computational setting, the synthetic DGP (Sec.~\ref{sec:experiments-poc}) and bike sharing experiment (Sec.~\ref{sec:experiments-bikesharing}) run in minutes. For the DeepAR experiments, training the DeepAR architecture across the three likelihoods takes approximately 2 hours in total, and the validation experiment takes approximately 10 hours (which can be reduced by using fewer instances). The Chronos experiment also takes approximately one hour. All experiments also include a “quick run” option, in which the code is executed at the cost of accuracy.

 \section{Impact Statement}\label{app:impact-statement}

This work focuses on a foundational post-hoc explanation methodology not tied to specific deployments. Any potential societal impacts, positive or negative, would therefore be indirect. For example, by improving transparency of probabilistic forecasters and decomposing predictive uncertainty into structurally distinct components, the framework can strengthen informed decisions in risk-sensitive applications such as energy planning or weather forecasting. On the negative side, unintended misuse where the model-relative Shapley attributions are misinterpreted as causal explanations could instill misplaced confidence in model reliability and have societal impacts when occurring in a safety-critical setting. However, this risk is not specific to this work, but inherent to the broader class of Shapley-based attribution methods.

\end{document}